\documentclass[letterpaper, 10 pt, conference]{ieeeconf}  

\IEEEoverridecommandlockouts                              

\usepackage{graphics} 
\usepackage{epsfig} 
\usepackage{mathptmx} 
\usepackage{times} 

\usepackage{amsmath} 
\usepackage{amssymb}  
\usepackage{amsthm} 

\usepackage{pifont}
\usepackage{tikz}

\newcommand{\rotbox}[1]{%
    \tikz[baseline=-0.5ex]
    \draw[rotate=#1] (-0.12,-0.06) rectangle (0.12,0.06);
}
\newcommand{\cmark}{\ding{51}}%
\newcommand{\xmark}{\ding{55}}%

\DeclareMathOperator{\Log}{Log}
\DeclareMathOperator{\Exp}{Exp}

\theoremstyle{definition}
\newtheorem*{preliminaries*}{Preliminaries}
\theoremstyle{plain}

\newtheorem{theorem}{Theorem}
\newtheorem{lemma}{Lemma}

\title{\LARGE \bf
Trajectory Bundle Method in SE(3) for Black-Box Fixed-Wing Aircraft Trajectory Optimization
}

\author{Matthew D. Osburn$^{1}$, Cameron K. Peterson$^{2}$, John L. Salmon$^{3}$
\thanks{$^{1}$Matthew D. Osburn is a PhD student in the Department of Electrical and Computer Engineering, 
        Brigham Young University,
Provo, UT 84602, USA
        {\tt\small osburnm@byu.edu}}%
\thanks{$^{2}$Cameron K. Peterson is a professor with the Department of Electrical and Computer Engineering, Brigham Young University,
Provo, UT 84602, USA
        {\tt\small cammy.peterson@byu.edu}}%
\thanks{$^{2}$John L. Salmon is a professor with the Department of Mechanical Engineering, Brigham Young University,
Provo, UT 84602, USA
        {\tt\small johnsalmon@byu.edu}}%
}

\begin{document}

\maketitle
\thispagestyle{empty}
\pagestyle{empty}

\begin{abstract}
Dynamically feasible trajectory optimization for rigid-body systems is naturally formulated on the special Euclidean group $\mathrm{SE}(3)$ but is challenging when dynamics are available only as black-box computations without derivatives. This paper formulates the Trajectory Bundle Method (TBM) for motion planning implicitly on $\mathrm{SE}(3)$. Bundles are constructed in the Lie algebra and propagated through nonlinear rigid-body dynamics using exponential and logarithmic maps, enabling derivative-free planning of non-Euclidean trajectories. We show that Euclidean TBM interpolation error is bounded quadratically by bundle diameter and extend this result to $\mathrm{SE}(3)$, where the bound additionally depends on a local Lipschitz constant of the Log map. Numerical experiments corroborate these bounds. Finally, we demonstrate SE(3) TBM by optimizing an acrobatic, collision-free fixed-wing maneuver through a rotated aperture without explicit models or derivatives of the vehicle dynamics, aerodynamics, or collision model.

\end{abstract}

\section{Introduction}

Trajectory optimization provides a systematic framework for generating dynamically feasible motion that satisfies state, control, and environmental constraints while optimizing a desired performance objective~\cite{kroger_-line_2010}. Unlike purely geometric path planning, trajectory optimization incorporates the system dynamics directly, allowing the resulting motion to be dynamically feasible and suitable for execution by a physical system~\cite{kelly_introduction_2017}. This capability is particularly important in robotics and unmanned systems, where vehicle dynamics, actuator limits, collision avoidance, and mission objectives must often be considered simultaneously~\cite{tedrake_underactuated_2024}.

Commonly, trajectory optimization is posed as a finite-horizon optimal-control problem~\cite{malyuta_convex_2022}. For nonlinear systems, computation time can dramatically increase because of the reliance on general nonlinear solvers and poor scaling as the number of variables increases~\cite{kelly_introduction_2017}. Methods such as sequential convex programming (SCP) have been developed to reduce computation time by repeatedly forming convex approximations of the optimization problem and iteratively improving the trajectory~\cite{schulman_motion_2014, bonalli_gusto_2019, mao_successive_2017}. These approaches are effective when accurate dynamics models and their derivatives are available, but can become more difficult to apply when derivative information is unavailable or expensive to obtain, as may occur with black-box or learned dynamics models \cite{manchester_derivative-free_2016, sukhija_gradient-based_2023}.

\begin{figure}[t]
    \centering
    \includegraphics[width=\linewidth]{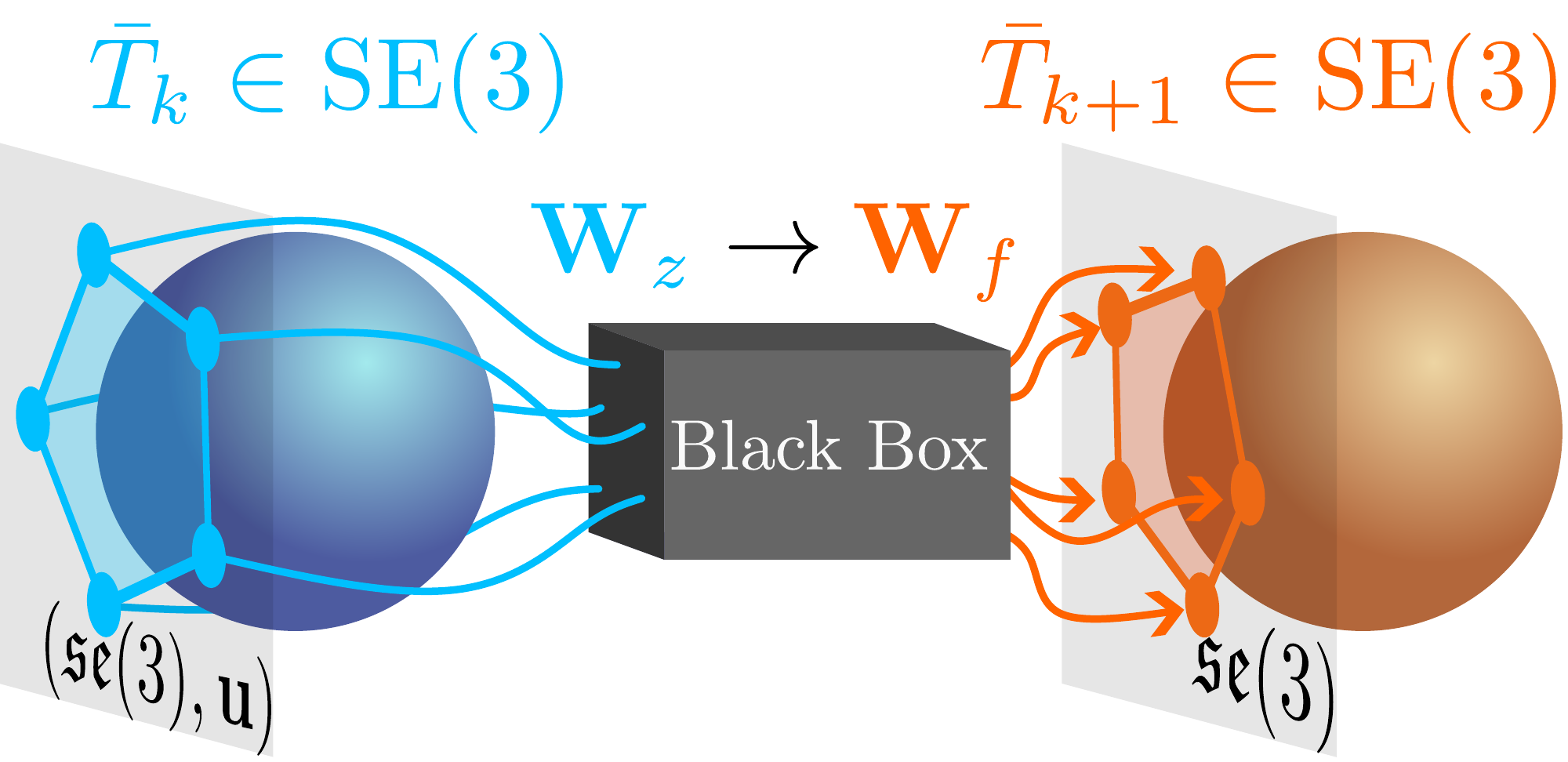}
    \caption{States in $\mathfrak{se}(3)$ and controls $\mathbf{u}$ are sampled to form an input matrix $\mathbf{W}_z$. This matrix is passed through a black-box system, such as a flight simulator, to form an output matrix $\mathbf{W}_f$. These two sampled matrices form a trajectory bundle that approximates system dynamics and enables sequential convex trajectory optimization.}
    \vspace{-5mm}
    \label{fig:tbm-illustration}
\end{figure}

The Trajectory Bundle Method (TBM), recently introduced by Tracy et al., addresses this limitation by replacing derivative-based local approximations with sampled trajectory bundles~\cite{tracy_trajectory_2025}. Most SCP schemes linearize the problem by using the Jacobians of the dynamics at every iteration. Instead, TBM evaluates the dynamics, costs, and constraints at nearby samples and uses linear interpolation to approximate non-linear relationships between input and output. The sampled rollouts can be evaluated in parallel and are used to construct convex subproblems without requiring the derivatives of the dynamics to be explicitly available or known. TBM has been demonstrated on double-integrator collision avoidance, quadrotor tracking, and cart-pole swing-up with learned neural dynamics~\cite{tracy_trajectory_2025}.

The existing TBM formulation constructs and interpolates bundles in Euclidean vector spaces or Euler angles. However, rigid-body motion does not have this structure: orientation evolves on the special orientation group, $\mathrm{SO}(3)$, and a full six-degree-of-freedom pose evolves on the special Euclidean group, $\mathrm{SE}(3)$~\cite{li_mathematical_1994}. Encoding orientation using local Euclidean coordinates introduces coordinate singularities and gimbal lock~\cite{hemingway_perspectives_2018} and limits the range of motion that can be planned in optimization. These issues are particularly important in planning problems that require large rotations or collision constraints that depend on vehicle attitude as well as position~\cite{liu_search-based_2018}.

We address this shortcoming of the TBM by sampling on a nominal $\mathrm{SE}(3)$ tangent space, propagating samples through potentially black-box rigid-body dynamics on $\mathrm{SE}(3)$, and mapping samples back to a tangent space where the local input-output relationship can be approximated. Although Lie-group methods are well established in estimation~\cite{barrau_intrinsic_2014}, control~\cite{zheng_integrated_2023, bullo_proportional_1995}, and trajectory optimization~\cite{schulman_motion_2014, watterson_trajectory_2018}, to the best of the authors' knowledge, the TBM has not previously been formulated directly on $\mathrm{SE}(3)$, nor has the interpolation error introduced by the $\mathrm{SE}(3)$ bundle approximation been characterized.

The contributions of this work are as follows:
\begin{enumerate}
\item A TBM formulation on $\mathrm{SE}(3)$ that constructs trajectory bundles in Lie-algebra tangent spaces and propagates samples through nonlinear black-box dynamics using the exponential and logarithmic maps. 
\item An interpolation-error characterization for TBM on $\mathrm{SE}(3)$, showing quadratic dependence on bundle diameter scaled by a local Lipschitz constant. 
\item Application of $\mathrm{SE}(3)$ TBM to collision-free acrobatic fixed-wing trajectory planning through a rotated aperture without explicit models or derivatives of the vehicle dynamics or collision computations.

\end{enumerate}


The remainder of this paper proceeds as follows: Section~\ref{sec:background} provides relevant background information; 
Section~\ref{sec:methods} develops the proposed $\mathrm{SE}(3)$ TBM formulation and characterizes the approximation error for the bundles; %
Section~\ref{sec:results} presents numerical results; and Section~\ref{sec:conclusion} concludes the paper. 

\section{Background}\label{sec:background}

\subsection{Rigid-Body Motion on $\mathrm{SE}(3)$}

The pose of a rigid body is represented by the special Euclidean group
\begin{equation}
T =
\begin{bmatrix}
R & \mathbf{p}\\
0_{1\times3} & I_{1\times1}
\end{bmatrix}
\in \mathrm{SE}(3),
\end{equation}
where $R\in\mathrm{SO}(3)$ is the orientation and $\mathbf{p}\in\mathbb{R}^3$ is the position. Unlike Euclidean vector spaces, $\mathrm{SE}(3)$ is a nonlinear manifold and does not admit globally valid vector addition.

The associated Lie algebra $\mathfrak{se}(3)$ is the tangent space of $\mathrm{SE}(3)$ at the identity. A vector $\boldsymbol{\xi}=[\boldsymbol{\rho}^\top,\boldsymbol{\varphi}^\top]^\top\in\mathbb{R}^6$ is mapped to $\mathfrak{se}(3)$ through the hat operator,
\begin{equation}
\boldsymbol{\xi}^\wedge =
\begin{bmatrix}
[\boldsymbol{\varphi}]_\times & \boldsymbol{\rho}\\
0_{1\times3} & 0_{1\times1}
\end{bmatrix},
\end{equation}
where $[\cdot]_\times$ is the skew-symmetric cross-product matrix, $\rho$ is the translational component and $\varphi$ is the rotational component of the pose. The inverse mapping of the hat operator is denoted by $(\cdot)^\vee$.

The exponential and logarithmic maps provide local mappings between the Lie algebra and group,
\begin{align}
\Exp(\boldsymbol{\xi}^\wedge)&=
\begin{bmatrix}
\Exp([\boldsymbol{\varphi}]_\times) & J_\ell(\boldsymbol{\varphi})\boldsymbol{\rho}\\
0_{1\times3} & I_{1\times1}
\end{bmatrix}, \\
\qquad
\Log(T)^\vee &=
\begin{bmatrix}
J_\ell(\boldsymbol{\varphi})^{-1}\mathbf{p}\\
\boldsymbol{\varphi}
\end{bmatrix},
\end{align}
where $J_\ell(\boldsymbol{\varphi})$ is the left Jacobian of $\mathrm{SO}(3)$.

Throughout this work, pose perturbations follow a right-perturbation convention. For a nominal pose $\bar T$, a tangent-space displacement $\boldsymbol{\xi}$ is mapped to $\mathrm{SE}(3)$ as
\begin{equation}
T=\bar T\Exp(\boldsymbol{\xi}^\wedge),
\qquad
\boldsymbol{\xi}=\Log(\bar T^{-1}T)^\vee.
\label{eq}
\end{equation}
Thus, $\Exp(\boldsymbol{\xi}^\wedge)$ maps a $\mathfrak{se}(3)$ local tangent-space displacement onto $\mathrm{SE}(3)$ , while $\Log(\bar T^{-1}T)^\vee$ maps from $\mathrm{SE}(3)$ back to the local coordinates on $\mathfrak{se}(3)$. See the grey tangent planes in Fig.~\ref{fig:tbm-illustration} for an illustrative example.

\subsection{Trajectory Bundle Method}

Consider the discrete-time trajectory optimization problem
\begin{align}  \label{eq:trajopt}
\min_{\mathbf{x}_{0:K},\mathbf{u}_{0:K-1}} \quad 
& \ell_f(\mathbf{x}_K)+\sum_{k=0}^{K-1}\ell_k(\mathbf{x}_k,\mathbf{u}_k)  \\
\text{s.t.}\quad &
\mathbf{x}_{k+1}=f_k(\mathbf{x}_k,\mathbf{u}_k), \nonumber \\
& h_k(\mathbf{x}_k,\mathbf{u}_k)=0, \qquad \nonumber\\
& g_k(\mathbf{x}_k,\mathbf{u}_k)\leq 0. \nonumber
\end{align}
where $\mathbf{x}_k\in\mathbb{R}^{n_x}$ is the trajectory state, $\mathbf{u}_k\in\mathbb{R}^{n_u}$ is the control input, $\ell_f$ is the terminal cost, $\ell_k$ is the stage cost, $f_k$ denotes the discrete system dynamics, and $h_k$ and $g_k$ represent equality and inequality constraints, respectively, all of which may be nonlinear. The trajectory contains $K+1$ knot points at which the trajectory is optimized and $K$ dynamic intervals.

The TBM constructs derivative-free local approximations of the nonlinear functions defined in Eq.~\eqref{eq:trajopt} using sampled function evaluations rather than Jacobians. Let $\Phi(\mathbf{z}):\mathbb{R}^{n_z}\rightarrow \mathbb{R}^{n_y}$ be the black-box function that maps from the input states and controls to an output value. Define the combined state-input vector as
\begin{equation}
\mathbf{z}_i=
\begin{bmatrix}
\mathbf{x}_i\\
\mathbf{u}_i
\end{bmatrix}
\in\mathbb{R}^{n_z},
\qquad n_z=n_x+n_u,
\end{equation}
where $n_x$ is the number of states in a sample, $n_u$ is the number of control inputs in a sample, $n_z$ is then the input dimension, and $n_y$ is the output dimension.
Let $n_s$ denote the number of samples in a bundle. The input and output sample matrices are
\begin{equation}
\mathbf{W}_z = \begin{bmatrix}\mathbf{z}_1 & \cdots & \mathbf{z}_{n_s}\end{bmatrix},
\qquad
\mathbf{W}_\Phi = \begin{bmatrix}\Phi(\mathbf{z}_1) & \cdots & \Phi(\mathbf{z}_{n_s})\end{bmatrix}.
\end{equation}
These two matrices form a trajectory bundle. Any point $\mathbf{z}$ in the convex hull of the input samples can be represented as
\begin{align}
\mathbf{z}&=\mathbf{W}_z\boldsymbol{\alpha},
\quad \boldsymbol{\alpha}\in A_{n_s} =\left\{\boldsymbol{\alpha}\in\mathbb{R}^{n_s}\mid
\alpha_i\geq0,\ \mathbf{1}^\top\boldsymbol{\alpha}=1\right\}. \label{eq:z_conv_sum}
\end{align}
TBM then approximates the corresponding function value using the output matrix as
\begin{equation}
\Phi(\mathbf{W}_z\boldsymbol{\alpha})
\approx
\mathbf{W}_\Phi\boldsymbol{\alpha}.
\label{eq:tbm_interp}
\end{equation}

Applying Eq.~\eqref{eq:tbm_interp} to the dynamics, cost residuals, and constraints produces a convex local subproblem in the interpolation weights $\boldsymbol{\alpha}$. For example, using the discrete dynamics function $f$ the multiple-shooting dynamics are approximated at knot point $k+1$ by
\begin{equation} \label{eq:dynamics_approx}
\mathbf{W}_{x,k+1}\boldsymbol{\alpha}_{k+1}
\approx
\mathbf{W}_{f,k}\boldsymbol{\alpha}_k.
\end{equation}

Typically, slack variables are introduced to the constraints in Eq.~\ref{eq:trajopt} to maintain feasibility and permit infeasible initial guesses. After each iteration, bundles are regenerated about the updated trajectory until convergence. Because TBM requires only sampled evaluations of the nonlinear functions, these evaluations can be parallelized and do not require analytic or automatic derivatives.

\section{Methods}\label{sec:methods}

This section outlines the process of formulating TBM on $\mathrm{SE}(3)$ and characterizes both the Euclidean bundle interpolation error and the additional effect of mapping that approximation onto $\mathrm{SE}(3)$.

\subsection{Trajectory Bundle Method on SE(3)} 

Let $r$ denote the TBM optimization iteration and $k$ the trajectory knot point. At each optimization iteration, trajectory bundles are constructed locally about the current nominal pose $\bar T_k^{(r)}$. For each knot point, $n_s$ pose and input perturbations $\boldsymbol{\xi}_{k,i}^{(r)}\in\mathbb{R}^6$ and $\delta\mathbf{u}_{k,i}^{(r)}\in\mathbb{R}^{n_u}$ are sampled in the local Lie-algebra coordinates of $\bar T_k^{(r)}$ and mapped onto $\mathrm{SE}(3)$ using
\begin{equation}
T_{k,i}^{(r)}
=
\bar T_k^{(r)}
\Exp\!\left[\left(\boldsymbol{\xi}_{k,i}^{(r)}\right)^\wedge\right].
\end{equation}
These perturbed states are then propagated through the black-box dynamics $f$ and mapped back to the Lie algebra associated with the next nominal pose at $k+1$,
\begin{equation}
\boldsymbol{\eta}_{k,i}^{(r)}
=
\Log\!\left[
\left(\bar T_{k+1}^{(r)}\right)^{-1}
 f_k\!\left(T_{k,i}^{(r)},
\bar{\mathbf{u}}_k^{(r)}+\delta\mathbf{u}_{k,i}^{(r)}\right)
\right]^\vee,
\end{equation} to build the sampled output of the system.  The sampled output is in the tangent space of the $k+1$ knot, allowing Eq.~\ref{eq:dynamics_approx} to be used to enforce dynamic feasibility on $\mathfrak{se}(3)$.
In this case, the dynamics $f$ plays the role of the black-box function $\Phi$ from the previous section.
Define the local bundle input sample as
\begin{equation}
\mathbf{z}_{k,i}^{(r)}=
\begin{bmatrix}
(\boldsymbol{\xi}_{k,i}^{(r)})^\top,
(\delta\mathbf{u}_{k,i}^{(r)})^\top
\end{bmatrix}^\top.
\end{equation}
The corresponding sample matrices are then
\begin{equation}
\mathbf{W}_{z,k}^{(r)}
=
\begin{bmatrix}
\mathbf{z}_{k,1}^{(r)} & \cdots & \mathbf{z}_{k,n_s}^{(r)}
\end{bmatrix},
\mathbf{W}_{f,k}^{(r)}
=
\begin{bmatrix}
\boldsymbol{\eta}_{k,1}^{(r)} & \cdots & \boldsymbol{\eta}_{k,n_s}^{(r)}
\end{bmatrix},
\end{equation}
 which define the trajectory bundle for the $k$th segment of the trajectory.  These matrices can be used to rewrite Eq.~\ref{eq:trajopt} as a convex subproblem within an SCP optimization routine. After solving the convex subproblem at iteration $r$, we update the nominal knot points and obtain $\bar T_k^{(r+1)}$. New perturbations are then sampled in these updated local coordinates and propagated to repeat the process. This scheme allows bundles to be reconstructed and propagated on $\mathrm{SE}(3)$ at every TBM iteration.

\subsection{Error Characterization}

This subsection characterizes the local interpolation error introduced by the TBM as a function of bundle diameter. We first establish a generic bound for Euclidean TBM interpolation and then extend the analysis to the proposed $\mathrm{SE}(3)$ formulation, where additional error arises from the Lie group's nonlinear geometry.

\begin{preliminaries*} 
Given an input sample matrix $\mathbf{W}_z$
with interpolation weights $\boldsymbol{\alpha}=[\alpha_1,\ldots,\alpha_{n_s}]^\top\in A_{n_s}$, let
$
S_z=\operatorname{conv}\{\mathbf{z}_1,\ldots,\mathbf{z}_{n_s}\}
$
be the convex hull of the input samples with diameter
${
\Delta
=\operatorname{diam}(S_z)
=\max_{\mathbf{z}_a,\mathbf{z}_b\in S_z}
\left\|\mathbf{z}_a-\mathbf{z}_b\right\|.
}$
Recall that any $\mathbf{z}\in S_z$ can be defined using a weigthed sum as shown in Eq.~\ref{eq:z_conv_sum}.

Let $\Phi:\mathbb{R}^{n_z}\rightarrow\mathbb{R}^{n_y}$ be the black-box mapping from inputs and control to outputs, and assume to be continuous and twice differentiable. For smooth dynamical systems, such as aircraft models, it is reasonable to assume that the second derivative of $\Phi$ is bounded by a constant $\kappa_\Phi$ over $S_z$.
Define the TBM approximation in terms of $\mathbf{z}$ as 
\begin{equation}\label{eq:tbm_approx_def}
\widetilde{\Phi}(\mathbf{z})
=\sum_{i=1}^{n_s}\alpha_i\Phi(\mathbf{z}_i).
\end{equation}

Define $D\Phi(\mathbf{z})$ as the Jacobian of $\Phi$ at the point $\mathbf{z}$.
For a unit direction $\delta\mathbf{z}\in Q=\{\mathbf{q}\in\mathbb{R}^{n_z} \mid \lVert \mathbf{q} \rVert =1 \}$, define the vector-valued directional second derivative as 
\begin{equation}
D^2\Phi(\mathbf{z})[\delta\mathbf{z},\delta\mathbf{z}]
=
\begin{bmatrix}
\delta\mathbf{z}^\top H_{\Phi_1(\mathbf{z})}\delta\mathbf{z}, \cdots,\delta\mathbf{z}^\top H_{\Phi_{n_y}(\mathbf{z})}\delta\mathbf{z}
\end{bmatrix}^\top.
\end{equation}

Define the Euclidean TBM approximation error as 
${
    \mathbf{e}_\Phi(\mathbf{z})
\triangleq
\Phi(\mathbf{z})-\widetilde{\Phi}(\mathbf{z}).
}$

\end{preliminaries*}

\begin{lemma}\label{lem:tbm_bound}

Assuming that the second derivative of $\Phi$ is bounded on $S_z$ in all directions $\delta\mathbf{z}$ by a constant $\kappa_\Phi$ such that
\begin{align}\label{eq:bounded_curvature}
    & \left\Vert D^2\Phi(\mathbf{z})[\delta\mathbf{z},\delta\mathbf{z}] \right\Vert \leq \kappa_\Phi, \quad \forall \mathbf{z}\in S_z, \forall \delta\mathbf{z}\in Q,
\end{align}
then the Euclidean TBM approximation error is bounded quadratically by the trajectory bundle diameter:
\begin{equation}\label{eq:error_bound}
    \left\Vert \mathbf{e}_\Phi(\mathbf{z})) \right\Vert \leq \frac{\kappa_\Phi}{2} \Delta^2.
\end{equation}
\end{lemma}

\begin{proof}
Let the displacement of sample $\mathbf{z}_i$ from the point $\mathbf{z}$ be
\begin{equation}
    \mathbf{d}_i=\mathbf{z}_i-\mathbf{z}.
\end{equation}

Multiplying by $\alpha_i$ and summing over the $n_s$ samples gives
\begin{align} \label{eq:barycentric_displacement}
\sum_{i=1}^{n_s}\alpha_i\mathbf{d}_i
&=
\sum_{i=1}^{n_s}\alpha_i\mathbf{z}_i
-
\mathbf{z}\sum_{i=1}^{n_s}\alpha_i 
=\mathbf{z}-\mathbf{z}=0.
\end{align} 

Expanding the Taylor-Series of each sampled output $\Phi(\mathbf{z}_i)$ about $\mathbf{z}$ and combining the second and higher-order terms into the remainder $\mathbf{r}_i$ gives
${
    \Phi(\mathbf{z}_i)
=
\Phi(\mathbf{z})
+D\Phi(\mathbf{z})\mathbf{d}_i
+\mathbf{r}_i.
}$


We multiply each $i$th sample's Taylor series by $\alpha_i$, sum over all samples, and use Eq.~\eqref{eq:z_conv_sum}, Eq.~\eqref{eq:tbm_approx_def}, and Eq.~\eqref{eq:barycentric_displacement} to give
\begin{align}\label{eq:summed_taylor}
\sum_{i=1}^{n_s}\alpha_i \Phi(\mathbf{z}_i) &=  \sum_{i=1}^{n_s}\alpha_i\Phi(\mathbf{z}) + 
\sum_{i=1}^{n_s}\alpha_iD\Phi(\mathbf{z})\mathbf{d}_i
+\sum_{i=1}^{n_s}\alpha_i\mathbf{r}_i \nonumber \\
\widetilde{\Phi}(\mathbf{z})
&=
\Phi(\mathbf{z})\sum_{i=1}^{n_s}\alpha_i
+D\Phi(\mathbf{z})
\sum_{i=1}^{n_s}\alpha_i\mathbf{d}_i
+\sum_{i=1}^{n_s}\alpha_i\mathbf{r}_i  \nonumber \\
\widetilde{\Phi}(\mathbf{z})&=
\Phi(\mathbf{z})
+\sum_{i=1}^{n_s}\alpha_i\mathbf{r}_i.
\end{align} 

Rearranging terms and expressing the result using the Euclidean TBM error yields
\begin{equation}\label{eq:simplified_taylor}
\mathbf{e}_\Phi(\mathbf{z})
=
-\sum_{i=1}^{n_s}\alpha_i\mathbf{r}_i.
 \end{equation}

By Taylor's remainder theorem and Eq.~\eqref{eq:bounded_curvature},
${\label{eq:bounded_remainder}
\left\|\mathbf{r}_i\right\|
\leq
\frac{\kappa_\Phi}{2}
\left\|\mathbf{d}_i\right\|^2.}
$
Taking the norm of Eq.~\eqref{eq:simplified_taylor} and applying the triangle inequality gives 
\begin{equation} \label{eq:taylor_error_norm}
\left\|\mathbf{e}_\Phi(\mathbf{z})\right\|
\leq
\frac{\kappa_\Phi}{2}
\sum_{i=1}^{n_s}\alpha_i
\left\|\mathbf{d}_i\right\|^2.
\end{equation}

Because $\mathbf{z},\mathbf{z}_i\in S_z$, by definition $\left\|\mathbf{d}_i\right\|\leq\Delta, \forall i\in \{1,\ldots,n_s$\}. Replacing $\left\|\mathbf{d}_i\right\|$ with $\Delta$ 
and factoring it out from the sum results in
\begin{equation}
    \left\|\mathbf{e}_\Phi(\mathbf{z})\right\|
\leq
\frac{\kappa_\Phi}{2}
\Delta^2\sum_{i=1}^{n_s}\alpha_i.
\end{equation}
Noting that $\alpha_i$ sums to 1, we obtain Eq.~\ref{eq:error_bound}. 
\end{proof}

This lemma is not specific to SE(3), and shows that the Euclidean TBM approximation error is quadratically bounded by the diameter of the samples used to construct the approximation. For the $\mathrm{SE}(3)$ approximation, the true propagated pose and the TBM-approximated pose are both represented relative to the next nominal pose  $\bar T_{k+1}$ as 
\begin{equation} \label{eq:nominal_relative_poses}
T_{\mathrm{true}}
=\bar T_{k+1}\Exp\!\left(\Phi(\mathbf{z})^\wedge\right), \text{ }
T_{\mathrm{TBM}}
=\bar T_{k+1}\Exp\!\left(\widetilde{\Phi}(\mathbf{z})^\wedge\right).
\end{equation}
The tangent-space error $\mathbf{e}_\Phi(\mathbf{z})$ is expressed in the Lie algebra associated with $\bar T_{k+1}$. Because SE(3) is non-commutative and non-linear, different Lie algebras will produce different errors between the same given poses.  Therefore, we want to compute the error between the TBM approximation and the true propagated pose directly.  We define
\begin{equation} \label{eq:true_error_def}
\boldsymbol{\varepsilon}_T
\triangleq
\Log\!\left(T_{\mathrm{true}}^{-1}T_{\mathrm{TBM}}\right)^\vee.
\end{equation} to be the error between the TBM approximation and the true pose in the true pose's frame of reference.

\begin{theorem}


Assume that the sampled rotations remain within a local region that does not cross the $\pi$ singularity of the logarithmic map, and that the relative rotation between $T_{\mathrm{true}}$ and $T_{\mathrm{TBM}}$ also remains away from $\pi$. Then there exists a local Lipschitz constant $L_\Psi$ such that
\begin{equation}\label{eq:etrue_bound}
\left\|\boldsymbol{\varepsilon}_T\right\|
\leq
L_\Psi\left\|\mathbf{e}_\Phi(\mathbf{z})\right\|
\leq
\frac{L_\Psi\kappa_\Phi}{2}\Delta^2.
\end{equation}
\end{theorem}

\begin{proof}
Substituting Eq.~\eqref{eq:nominal_relative_poses} into Eq.~\eqref{eq:true_error_def} gives
\begin{equation}
\left\|\boldsymbol{\varepsilon}_T\right\|
=
\left\|
\Log\!\left[
\Exp\!\left(-\Phi(\mathbf{z})^\wedge\right)
\Exp\!\left(\widetilde{\Phi}(\mathbf{z})^\wedge\right)
\right]^\vee
\right\|.
\end{equation}
Define
${
\Psi(\mathbf{a},\mathbf{b})
=
\Log\!\left[
\Exp(-\mathbf{a}^\wedge)
\Exp(\mathbf{b}^\wedge)
\right]^\vee,
\mathbf{a},\mathbf{b}\in\mathbb{R}^6.
}$
Then
$ {\label{eq:error_true_psi_no_bound}
\left\|\boldsymbol{\varepsilon}_T\right\|
=
\left\|
\Psi\!\left(\Phi(\mathbf{z}),\widetilde{\Phi}(\mathbf{z})\right)
\right\|.}
$
Within the assumed local region the difference between two poses will be finite and $\Psi$ is locally Lipschitz in its second argument, so
\begin{equation}
\left\|
\Psi(\mathbf{a},\mathbf{b})-\Psi(\mathbf{a},\mathbf{a})
\right\|
\leq
L_\Psi\left\|\mathbf{b}-\mathbf{a}\right\|.
\end{equation}
By definition,
$
\Psi(\mathbf{a},\mathbf{a})
=
\Log\!\left[
\Exp(-\mathbf{a}^\wedge)\Exp(\mathbf{a}^\wedge)
\right]^\vee
=0,
$
and 
\begin{align}
\left\|\boldsymbol{\varepsilon}_T\right\|
&\leq
L_\Psi
\left\|
\widetilde{\Phi}(\mathbf{z})-\Phi(\mathbf{z})
\right\| \leq L_\Psi\left\|\mathbf{e}_\Phi(\mathbf{z})\right\|.
\label{eq:lip_bound}
\end{align}
Applying Lemma~\ref{lem:tbm_bound} yields Eq.~\eqref{eq:etrue_bound}.
\end{proof}

 For close pose transitions between knot points the Log map's Lipschitz constant should remain well-behaved. For certain nonsmooth functions, such as signed-distance functions, we conjecture that the variation of the subgradient might be bounded by a finite constant, extending the results of Lemma~\ref{lem:tbm_bound}. Establishing this rigorously is left for future work.


\section{Results}\label{sec:results}

The results first validate the theoretical TBM error bounds and characterize how approximation error scales with bundle size. We then evaluate the proposed $\mathrm{SE}(3)$ formulation in trajectory-optimization experiments involving acrobatic maneuvers and collision-avoidance constraints.

\subsection{Empirical Error Validation} \label{sec:error-scaling-results}

We validate the local interpolation bound for $\mathrm{SE}(3)$ TBM by measuring the sampled dynamics error of a single propagation step.  For each
bundle scale factor $\gamma_b$, a regular simplex was constructed in the
12-dimensional coordinate
$\mathbf{z}_{k,i}=[\boldsymbol{\xi}_{k,i},\boldsymbol{\nu}_{k,i}]^\top,$
where $\boldsymbol{\xi}\in\mathbb{R}^6$ is the local pose perturbation and $\boldsymbol{\nu}\in\mathbb{R}^6$ is the body-frame twist.  The simplex has
13 vertices, as expected for a full-dimensional simplex in 12 dimensions.
At each vertex, the dynamics
\begin{equation}
\boldsymbol{\xi}_{k+1}
=f_{\mathrm{kin}}(\boldsymbol{\xi}_k,\boldsymbol{\nu}_k)
=
\Log\!\left[
\Exp(\boldsymbol{\xi}_k^\wedge)
\Exp((t_s\boldsymbol{\nu}_k)^\wedge)
\right]^\vee
\end{equation}
were evaluated, where $t_s$ represents the integration timestep. 


The experiment used a step size $t_s=0.1$. The interpolation weights $\boldsymbol{\alpha}$ that generate $\mathbf{z}=\mathbf{W}_z\boldsymbol{\alpha}$ and the TBM approximation $f_{\mathrm{kin}}(\mathbf{W}_z\boldsymbol{\alpha})\approx\mathbf{W}_f\boldsymbol{\alpha}$ were sampled from a uniform Dirichlet distribution. The sweep used 25 bundle scale factors $\gamma_b$ from $10^{-3}$ to $4\times10^{-1}$. Because the simplex geometry is fixed and uniformly scaled, its diameter satisfies $\Delta=\gamma_b\Delta_0$, where $\Delta_0$ is the diameter of the unscaled simplex. At each scale, 5000 interior query points were sampled and propagated, making 125000 error samples in total. 


\begin{figure}[t]
    \centering
    \includegraphics[width=0.75\linewidth]{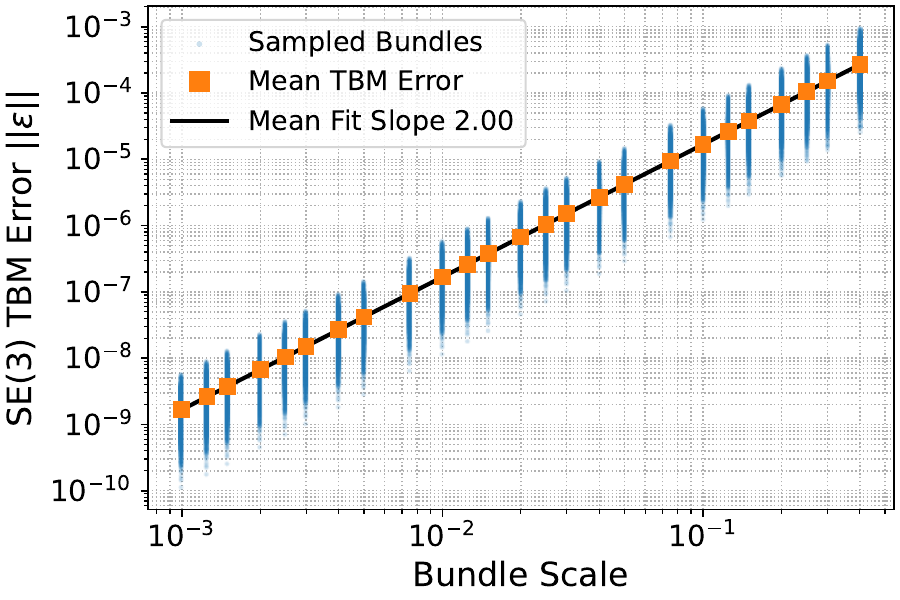}
    \caption{Log-log validation of the $\mathrm{SE}(3)$ TBM error. Blue points show sampled bundle errors $|\boldsymbol{\varepsilon}_T|_2$ in the true propagated pose frame, orange squares show the mean error at each bundle scale $\gamma_b$, and the black line is a least-squares fit to these means. The fitted slope of 2.00043 confirms the predicted quadratic dependence on bundle diameter.}
    \vspace{-5mm}
    \label{fig:error-scaling}
\end{figure}

As shown in Figure~\ref{fig:error-scaling}, the fitted log-log slope of the mean error was 2.00043 with a maximum relative residual of $0.0037$. A median-error fit produced a slope of 2.00047 with a maximum relative residual of $0.0049$, confirming that both statistics closely follow the predicted scaling. Since $\Delta\propto\gamma_b$ for the uniformly scaled simplex, these results are consistent with the $\Delta^2$ interpolation-error bound and numerically validate the quadratic error growth predicted for $\mathrm{SE}(3)$ TBM.

\subsection{Rotated Aperture Experiment}

Bundle-scale effects were evaluated by planning a collision-free fixed-wing maneuver through a rotated rectangular aperture from a dynamically infeasible initial trajectory. This stress test exposes the trade-off between small bundles, which limit motion between iterations, and large bundles, which increase approximation error. PyFlyt~\cite{tai2023pyflyt} provided black-box aerodynamics, rigid-body dynamics, and collision computations, and CVXPY solved the resulting optimization problem.

The $4\times1.5$ m aperture was rotated by $30^\circ$, $60^\circ$, and $90^\circ$, while the aircraft frontal profile measured $2.2\times0.45$ m. Increasing aperture rotation therefore required increasingly aggressive attitude changes to maintain clearance. An $\mathrm{SE}(3)$ state representation was used to ensure that Euler-angle singularities or gimbal-lock artifacts were avoided during optimization.

Trajectories used $K+1=41$ knots separated by $0.1$ s over a 4 s horizon, with the aperture crossed at the midpoint knot $k_g=20$. The initial guess was a straight trim trajectory whose middle six knots were rotated with the aperture. This produced an approximately collision-free but dynamically infeasible trajectory, requiring the optimizer to restore dynamic feasibility while preserving clearance.

Bundles used states $\mathbf{x}=[\boldsymbol{\xi},\dot{\boldsymbol{\xi}}]^\top$ and controls $\mathbf{u}=[u_{\mathrm{roll}},u_{\mathrm{pitch}},u_{\mathrm{yaw}},u_{\mathrm{thrust}}]^\top$.  These control channels are passed through PyFlyt's control surface mixer to produce flight commands, which the optimizer considers part of the black-box dynamics. State and control coordinates were sampled in positive and negative directions, except for non-negative throttle, yielding $n_s=192$ state-control pairs per knot and 7872 per iteration. PyFlyt's control-surface mixer remained within the black-box dynamics. The scale factor $\gamma_b$ uniformly scaled bundle offsets and therefore the bundle diameter $\Delta$.

For each iteration, the trajectory was obtained by solving
\begin{align}
\min_{\boldsymbol{\alpha},\,\mathbf{s}_{\mathrm{dyn}},\,\mathbf{s}_{\mathrm{col}}}\quad
    & \mu_{\mathrm{gate}}
      \left\|
      \mathbf{W}_{x,k_g}\boldsymbol{\alpha}_{k_g}
      -\mathbf{x}_{\mathrm{gate}}
      \right\|_2^2 \\
    &+ \mu_{\mathrm{dyn}}
      \sum_{k=0}^{K-1}\left\|\mathbf{s}_{\mathrm{dyn},k}\right\|_2^2
    + \mu_{\mathrm{col}}
      \sum_{k=0}^{K}\left\|\mathbf{s}_{\mathrm{col},k}\right\|_2^2  \nonumber \\
\text{s.t.}\quad
    & \mathbf{W}_{x,k+1}\boldsymbol{\alpha}_{k+1}
      +\mathbf{s}_{\mathrm{dyn},k}
      =\mathbf{W}_{f,k}\boldsymbol{\alpha}_k, \nonumber\\
    & \mathbf{W}_{c,k}\boldsymbol{\alpha}_k
      +\mathbf{s}_{\mathrm{col},k}
      \geq c_{\min}, \nonumber
\end{align}
where $c_{\min}=0.05$ m. The first term penalizes deviation of the trajectory at the aperture from the desired state with weight $\mu_{\mathrm{gate}}=10^2$. The slack variables $\mathbf{s}_{\mathrm{dyn},k}$ and $\mathbf{s}_{\mathrm{col},k}$ penalize inconsistency with the system dynamics and collision constraints and are weighted by $\mu_{\mathrm{dyn}}=10^6$ and $\mu_{\mathrm{col}}=10^4$, respectively. The matrices $\mathbf{W}_{x,k}$, $\mathbf{W}_{f,k}$, and $\mathbf{W}_{c,k}$ contain the $k$th knot's sampled states, propagated dynamics outputs, and sampled collision clearances, respectively. The optimization was considered to have converged if the dynamics residual was less than $10^{-4}$, the collision clearance was greater than $0.1$ m, and the improvement in objective was less than $10^{-2}$ for three iterations.

We validated these trajectories using a pose-tracking aircraft controller in PyFlyt. Figure~\ref{fig:se3-gate-90-3d} shows the $\gamma_b=1$ solution for the $90^\circ$ aperture. PyFlyt simulated physics at 180 Hz, 18 times faster than the 10 Hz discretization used during planning. Figure~\ref{fig:se3-gate-90-controller} shows the controller tracking over time for the maneuver. This result provides an independent simulation and controller check that the trajectory planned with black-box dynamics at a time step of 0.1 s is executable when evaluated at the higher simulation rate.

\begin{figure}[!ht]
    \centering
    \includegraphics[width=\linewidth]{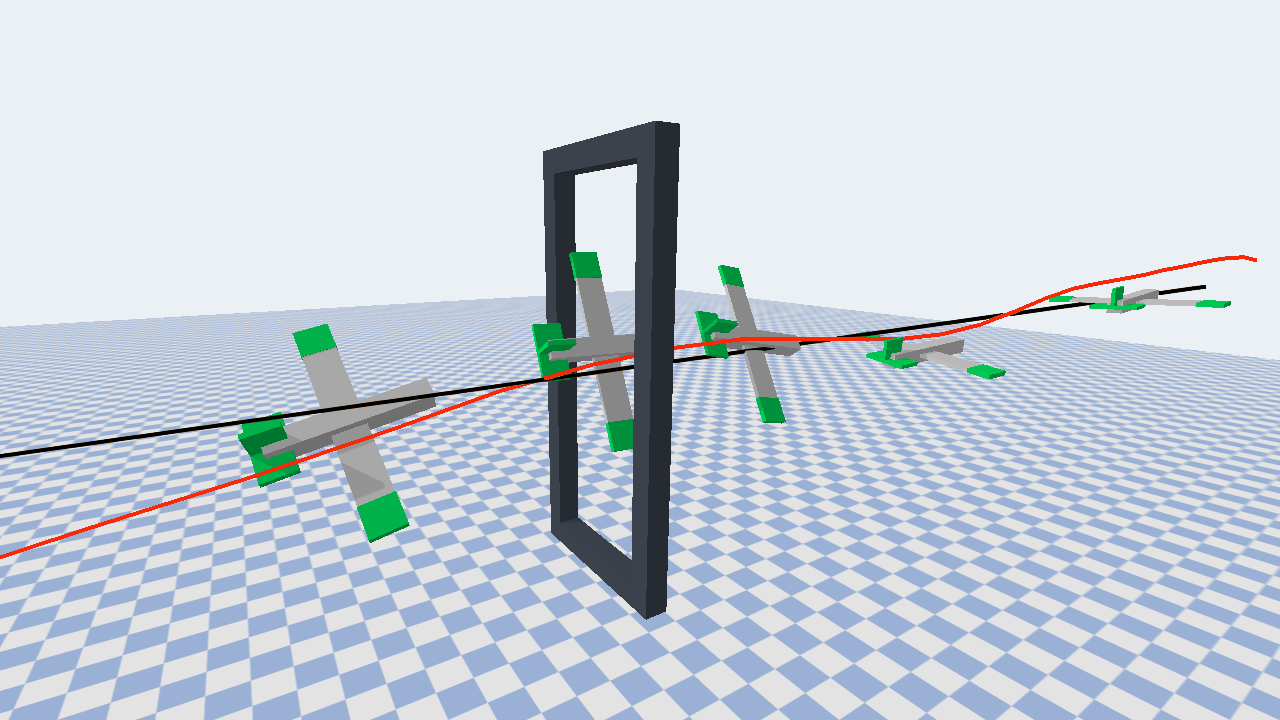}
    \caption{Closed-loop simulation of the $\gamma_b=1$ trajectory through the $90^\circ$ rotated aperture. The infeasible initial trajectory is shown in black and the optimized trajectory in red. The aircraft is gray with green wing tips, and the aperture is dark gray. The controller runs at 180 Hz, 18 times the temporal resolution used for trajectory optimization.}
    \vspace{-3mm}
    \label{fig:se3-gate-90-3d}
\end{figure}

\begin{figure}[!ht]
    \centering
    \includegraphics[width=1.0\linewidth]{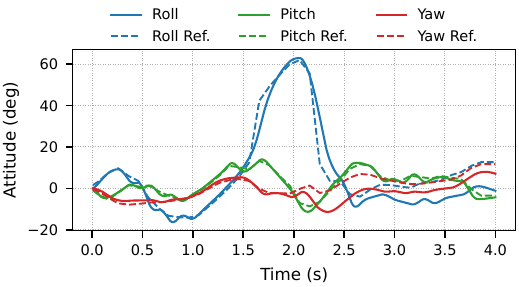}
    \caption{Controller tracking attitude over time for the $\gamma_b=1$ trajectory through the $90^\circ$ rotated aperture. The planned trajectory was flyable at the increased simulation resolution.}
    \vspace{-5mm}
    \label{fig:se3-gate-90-controller}
\end{figure}

\begin{table}[h]
\caption{SE(3) Gate Scale Sweep Results} \label{tab:se3-gate-scale-sweep-fixed-combined}
\begin{tabular}{|c|c|c|c|c|c|}
\hline
Gate & Metric & 0.5 & 1 & 2 & 4 \\
\hline
 & Converged & \cmark & \cmark & \cmark & \cmark \\
 & Iterations & 11 & \textbf{10} & \textbf{10} & 11 \\
$\textbf{30}^{\circ} $ & Objective & 1472 & 2394 & 2417 & \textbf{1443} \\
 \rotbox{30}& Dynamics violation & 0 & 0 & 0 & 0 \\
 & Clearance & \textbf{0.4242} & 0.4123 & 0.3881 & 0.3992 \\
 & Control Success? & \cmark & \cmark & \cmark & \cmark \\
\hline
 & Converged & \xmark & \cmark & \cmark & \cmark \\
 & Iterations & 7 & \textbf{20} & \textbf{20} & 26 \\
$\textbf{60}^{\circ}$ & Objective & 1208 & 2102 & 2002 & \textbf{1834} \\
 \rotbox{60}& Dynamics violation & 0.4554 & 0 & 0 & 0 \\
 & Clearance & 0.4501 & 0.1149 & 0.1265 & \textbf{0.1422} \\
 & Control Success? & - & \cmark & \cmark & \cmark \\
\hline
 & Converged & \xmark & \cmark & \cmark & \xmark \\
 & Iterations & 41 & 30 & \textbf{24} & 25 \\
$\textbf{90}^{\circ}$ & Objective & 1616 & \textbf{1676} & 1715 & 1613 \\
 \rotbox{90}& Dynamics violation & 0.09513 & 0 & 0 & 0.2604 \\
 & Clearance & 0.1035 & 0.1009 & \textbf{0.1117} & 0.1098 \\
 & Control Success? & - & \cmark & \cmark & - \\
\hline
\end{tabular}
\vspace{-5mm}
\end{table}

Table~\ref{tab:se3-gate-scale-sweep-fixed-combined} shows the effect of bundle scale as the required maneuver becomes increasingly aggressive. The reported objective is only the first term of the cost function, with the other terms reported instead as a dynamics residual and the distance between the aircraft and the aperture. Validation in simulation was only applied to trajectories that converged. 

For the $30^\circ$ aperture, all tested bundle scales converged to dynamically feasible trajectories and successfully passed the subsequent simulation test. The number of iterations was also nearly independent of bundle scale, ranging only from 10 to 11 iterations; however, objectives ranged more widely. While objectives varied between all runs, all trajectories were successful in simulation.  

The effect of bundle size is more apparent for the $60^\circ$ aperture. With $\gamma_b=0.5$, the optimizer terminates with a dynamics violation of $0.4554$, while $\gamma_b=1$, $2$, and $4$ all converge to feasible trajectories with zero dynamics residual. Because the initial trajectory is dynamically infeasible, the bundle must be large enough to move sufficiently far from the initial guess over the first few iterations to recover feasibility. The more aggressive maneuver also requires roughly twice as many iterations as the $30^\circ$ case.

The $90^\circ$ aperture illustrates the opposing limitation of excessively large bundles. The small bundle with $\gamma_b=0.5$ is again unable to overcome the initial dynamics violation. In contrast, $\gamma_b=1$ and $\gamma_b=2$ both converge to dynamically feasible trajectories and subsequently succeed in simulation. Increasing the bundle scale to $\gamma_b=4$, however, results in a dynamics violation of $0.2604$ and a failed simulation. In this case, the bundle was large enough to fix dynamic infeasibility, only to over-refine the objective and create new infeasibilities in the next iteration.


As a rule of thumb, we found that the scales of the bundle axes should be set to a quantity that can reasonably be achieved during the integration window between knots. For example, the position samples should be separated by a distance that could reasonably be covered by the vehicle in 0.1 seconds at trim conditions. The bundle scales $\gamma_b\in\{1,2\}$ adhered roughly to this rule of thumb with respect to the real aircraft dynamics and, as a result, produced feasible and flyable results even when making large maneuvers.


\section{Conclusion}\label{sec:conclusion}


These results establish a geometric formulation of TBM that respects the nonlinear structure of rigid-body motion while retaining a practical interpretation of bundle size as a tuning parameter. Future work includes extending the approach to incorporate uncertainty and disturbances, and evaluating the method on higher-fidelity and hardware-based platforms.

\addtolength{\textheight}{-12cm}   




\section*{ACKNOWLEDGMENT}

This project was supported by NSF SBIR Phase 2 Award Number 2404858 and 4D Avionic Systems, LLC.


The authors used ChatGPT to improve the paper’s readability and wording, then thoroughly reviewed and edited the content for accuracy and coherence. The authors take full responsibility for the content of the published article.


\bibliographystyle{IEEEtran}
\bibliography{references}

\end{document}